\documentclass[final]{opt2026} 

\usepackage{mathtools}
\usepackage{booktabs}
\usepackage{colortbl}
\definecolor{blockgray}{gray}{0.90}

\DeclareMathOperator*{\Var}{Var}
\DeclareMathOperator*{\ESS}{ESS}

\newcommand{\best}[1]{\textbf{#1}}
\newcommand{\snd}[1]{\underline{#1}}
\newcommand{\pms}[1]{{\scriptsize$\,\pm#1$}}
\newcommand{\sg}{\operatorname{sg}}

\providecommand{\proofname}{Proof}
\renewenvironment{proof}[1][\proofname]{%
  \par\smallskip\noindent{\bfseries #1\ }\normalfont}{%
  \hfill\jmlrBlackBox\par\smallskip}

\newtheorem{gdefn}{Definition}
\newtheorem{gprop}{Proposition}

\title[TETPO: Smooth Variance Control via Bridge IS]{Target--ESS Tempered Policy Optimization (TETPO):\\ Smooth Variance Control via Bridge Importance Sampling}

\optauthor{%
\Name{Shayan {Mohajer Hamidi}} \Email{smohajer@linkedin.com}\\
\Name{Kayhan Behdin} \Email{kbehdin@linkedin.com}\\
\Name{Wenhui Zhu} \Email{wenhzhu@linkedin.com}\\
\Name{Aida Rahmattalabi} \Email{arahmattalabi@linkedin.com}\\
\Name{Jincheng Cao} \Email{jincao@linkedin.com}\\
\Name{Jelena Markovic-Voronov} \Email{jmarkovicvoronov@linkedin.com}\\
\Name{Alborz Geramifard} \Email{agf@linkedin.com}\\
\addr LinkedIn}

\begin{document}

\maketitle

\begin{abstract}%
Asynchronous reinforcement-learning post-training of large language models is
computationally efficient but inherently off-policy: the learner updates on
rollouts generated by a stale policy, while correcting this mismatch requires
importance weights that can be extremely heavy-tailed. A small number of
trajectories may then dominate an update, collapsing the effective sample size
(ESS) and increasing the variance and instability of the resulting gradient
estimates. Existing stabilizers apply fixed interventions regardless of the
realized severity of weight concentration, either discarding selected samples or
reweighting them without adapting the correction strength. We introduce
\emph{Target--ESS Tempered Policy Optimization} (TETPO), which replaces each
importance weight $\rho$ by $\rho^\alpha$ for a data-adaptive exponent
$\alpha\in[0,1]$: $\alpha=1$ preserves the original importance weights, whereas
$\alpha=0$ assigns uniform coefficients. For each prompt group, TETPO selects the
largest $\alpha$ whose tempered weights satisfy a target ESS fraction $\kappa$,
applying little tempering to healthy groups and stronger tempering when weight
degeneracy is severe. At the population level, the tempered coefficients are exact
importance weights for a geometric bridge between the sampler and learner
policies. The exponent is efficiently found by one-dimensional bisection, while
the resulting weights satisfy deterministic bounds on normalized-weight
dispersion and weighted-gradient norm that are independent of policy drift,
together with a fixed-$\alpha$ bound on bridge-target discrepancy. Power
tempering also introduces no explicit zero-weight mask. On mathematical-reasoning
benchmarks under high staleness, TETPO outperforms strong on- and off-policy
baselines. 
\end{abstract}


\section{Introduction and Motivation}
\label{sec:async}

Modern reinforcement-learning (RL) post-training of large language models (LLMs)
is dominated by the cost of \emph{rollout generation}: producing several long
completions $y \sim \mu(\cdot \mid x)$ per prompt $x$, scoring them, and taking
policy-gradient steps on the resulting $(x, y, R)$ triples. In strictly
synchronous schedules the learner sits idle during sampling. Asynchronous
schedules~\citep{noukhovitch2024arxiv,fu2025areal,wu2025llamarl,piche2025pipelinerl}
break this coupling: a sampler generates from a policy snapshot several gradient
steps in the past while the learner trains on rollouts as they arrive, giving
end-to-end speedups of $2$--$3\times$ over synchronous training~\citep{piche2025pipelinerl}.

The cost is that training becomes \emph{off-policy}. Let $\pi_\theta = \pi_{\theta_t}$
be the learner at step $t$ and $\mu = \pi_{\theta_{t'}}$ a sampler snapshot at
step $t' = t - k$. Assuming $\pi_\theta(\cdot\mid x)\ll\mu(\cdot\mid x)$ and
$\theta$-independent $R$, importance sampling (IS) yields the gradient identity
\begin{equation}
\nabla_\theta J(\theta)
=
\mathbb{E}_{x \sim \mathcal{D},\, y \sim \mu(\cdot \mid x)}
\!\Big[\, w(x,y)\, \big(R(x,y)-b(x)\big)\, \nabla_\theta \log \pi_\theta(y \mid x) \,\Big],
\quad
w(x,y) \triangleq \frac{\pi_\theta(y \mid x)}{\mu(y \mid x)},
\label{eq:offpolicy}
\end{equation}
for any $y$-independent baseline $b(x)$. The group-relative advantages used in the
practical surrogate are finite-sample variance-reduction terms; our results below
concern the tempered weights conditional on the realised advantages, not
unbiasedness of the full clipped update.
But $w(x,y) = \prod_{t=1}^{T} \pi_\theta(y_t \mid \cdot)/\mu(y_t \mid \cdot)$ is
a product of $T$ per-token ratios, with $T$ in the thousands. Even small
per-token shifts compound multiplicatively, so $w$ is extremely heavy-tailed:
the \emph{curse of the horizon}~\citep{liu2018curse,liu2020understanding}. In
practice a few trajectories dominate each update, and the \emph{effective sample
size}
\begin{equation}
\mathrm{ESS}(\mathbf w) \triangleq \frac{\big(\sum_{i=1}^B w_i\big)^2}{\sum_{i=1}^B w_i^2}
\in [1, B]
\label{eq:ess}
\end{equation}
collapses below the batch size $B$. When ESS collapses, weight concentration can
sharply increase the variance and instability of the resulting gradient estimates,
the learner-to-sampler KL spikes, and reward can collapse. The central question is
therefore:
\emph{how do we control the dispersion of heavy-tailed IS weights without losing
so much information that learning stalls?}

\paragraph{Existing stabilizers.}
Two recent methods take complementary approaches.
\emph{M2PO}~\citep{m2po2026} works per token: with ratio
$r_{i,t} = \pi_\theta(y_{i,t}\mid\cdot)/\mu(y_{i,t}\mid\cdot)$ and second-moment
proxy $M_{i,t}^{(2)} = (\log r_{i,t})^2$, it greedily zeroes the highest-$M^{(2)}$
tokens in a trust region until the mean $M^{(2)}$ falls below a threshold
$\tau_{M_2}$, yielding a masked surrogate
$\mathcal{L}_{\text{M2PO}} = -\tfrac{1}{N}\sum_{i,t} m_{i,t}\, r_{i,t}\, A_i$
with $m_{i,t}\in\{0,1\}$.
\emph{GEPO}~\citep{gepo} works at the sequence level and never masks: rather than
normalizing each likelihood ratio by the sum of the ratios in the batch, it
divides the learner sequence-likelihood by a \emph{group-expectation} denominator,
i.e.\ an estimate $\widehat{\mathbb E}_{\text{grp}}[\pi_\theta]$ of the expected
learner likelihood over the group, giving a group-expectation-normalized weight
$\hat{w}^{\text{GEPO}}_i = \pi_\theta(y_i) / \widehat{\mathbb E}_{\text{grp}}[\pi_\theta]$
(computed in log space for stability), so every sample contributes and nothing is
discarded.

\paragraph{Failure modes.}
Each buys stability at a cost that becomes acute as lag $k$ grows.
\textbf{(i) M2PO collapses under heterogeneous staleness.} The mask is
discontinuous in $r_{i,t}$; when prompts in a batch are sampled at different
lags, the $M^2$ statistic is dominated by high-lag tokens and the greedy mask
zeroes a large fraction of the trust region, sometimes entirely, so that under
mixed-staleness queues M2PO's trusted-token fraction becomes erratic and the run
can crash. In our $k=32$ runs
M2PO peaks early then \emph{declines} on every benchmark.
\textbf{(ii) GEPO has no variance-control knob.} Its group-expectation
normalization is fixed and has no tunable strength, so when a batch is dominated
by a few heavy sequences it passes through highly concentrated weights without an
explicit ESS constraint; at $k=32$ GEPO suffers a catastrophic reward drop near
step~$200$.

\paragraph{Our Approach.}
We want an estimator that \emph{(a)} \textbf{adapts} its correction strength to
the groupwise ESS: aggressive when ESS is collapsing, dormant when
healthy; \emph{(b)} is \textbf{smooth} in the IS ratios (no discrete masking);
\emph{(c)} provides \textbf{a bound on the fixed-$\alpha$ bridge-target
discrepancy}; and \emph{(d)}
\textbf{cannot collapse}, so weights stay strictly positive even under an
adversarial batch which motivate TETPO.

\section{Methodology: Target--ESS Tempered Policy Optimization}
\label{sec:tetpo}

We follow the off-policy LLM-RL setting of M2PO~\citep{m2po2026}. A minibatch
contains $M$ prompts $x_1,\dots,x_M$ with $G$ completions each; for prompt $x_m$
the completions $y_{m,1},\dots,y_{m,G}$ are drawn from the recorded behaviour
policy $\mu_m(\cdot)\coloneqq\mu(\cdot\mid x_m)$, and the sequence-level ratio is
\begin{equation}
\rho_{m,i} \coloneqq \frac{\pi_\theta(y_{m,i}\mid x_m)}{\mu_m(y_{m,i})}
= \exp\!\Big(\textstyle\sum_{t} \log\tfrac{\pi_\theta(y_{m,i,t}\mid x_m,y_{m,i,<t})}{\mu_m(y_{m,i,t}\mid x_m,y_{m,i,<t})}\Big).
\end{equation}
Because the bridge normaliser $Z_\alpha(x)$ (Theorem~\ref{thm:bridge}) is
prompt-dependent, we normalise and select the tempering exponent \emph{within each
prompt group}, not across the pooled minibatch (which would implicitly reweight the
prompt distribution).

\begin{gdefn}[Groupwise tempered weights and ESS]
\label{def:tempered}
For $\alpha \in [0,1]$ and each group $m$,
\begin{equation}
\bar w_{m,i}(\alpha) \coloneqq \rho_{m,i}^{\,\alpha},
\qquad
\tilde w_{m,i}(\alpha) \coloneqq \frac{\bar w_{m,i}(\alpha)}{\frac{1}{G}\sum_{j=1}^{G} \bar w_{m,j}(\alpha)},
\end{equation}
so $\frac{1}{G}\sum_i \tilde w_{m,i}(\alpha) = 1$, and the groupwise effective sample
size is
\[
\ESS_m(\alpha) \coloneqq \frac{(\sum_i \rho_{m,i}^{\alpha})^2}{\sum_i \rho_{m,i}^{2\alpha}} \in [1, G].
\]
\end{gdefn}

\begin{gdefn}[Groupwise target-ESS exponent]
\label{def:alpha-star}
Given a target ESS fraction $\kappa \in (0, 1]$,
\begin{equation}
\alpha_m^\star \coloneqq \max\big\{\, \alpha \in [0, 1] \;:\; \ESS_m(\alpha) \geq \kappa\, G \,\big\}.
\label{eq:alphastar}
\end{equation}
Since $\rho_{m,i}^{0}=1$ gives $\ESS_m(0)=G\ge\kappa G$, the feasible set is
non-empty; as $\alpha\mapsto\ESS_m(\alpha)$ is continuous and non-increasing
(Theorem~\ref{thm:alpha-exist}) the maximum exists with $\ESS_m(\alpha_m^\star)\ge\kappa G$
for \emph{every} group. We search all of $[0,1]$ (no floor $\alpha_{\min}$) by
$1$-D bisection.
\end{gdefn}

\paragraph{Stop-gradient surrogate.}
Let $\ell_{m,i}(\theta)\coloneqq\log\pi_\theta(y_{m,i}\mid x_m)$ and let $\sg[\cdot]$
be the stop-gradient operator. The exponent $\alpha_m^\star$ and weights
$\tilde w_{m,i}(\alpha_m^\star)$ are computed from detached ratios. Define the
straight-through surrogate ratio
\begin{equation}
r_{m,i}(\theta)\coloneqq
\sg\!\big[\tilde w_{m,i}(\alpha_m^\star)\big]\,
\exp\!\big(\ell_{m,i}(\theta)-\sg[\ell_{m,i}(\theta)]\big),
\end{equation}
whose forward value is $\tilde w_{m,i}(\alpha_m^\star)$ and whose gradient is
$\tilde w_{m,i}(\alpha_m^\star)\,\nabla_\theta\ell_{m,i}(\theta)$; at
$\alpha_m^\star=0$ the method reduces to a uniform-coefficient, GRPO-style score
update on the stale samples (not an on-policy update, since the samples still come
from $\mu_m$). With group-relative advantage $A_{m,i}$,
\begin{equation}
\mathcal{L}_{\text{TETPO}}(\theta)
=
-\,\frac{1}{MG}\sum_{m=1}^{M}\sum_{i=1}^{G}
\min\!\Big(
  r_{m,i}(\theta)\, A_{m,i},\;
  \mathrm{clip}\big(r_{m,i}(\theta),\, 1-\varepsilon_{\text{lo}},\, 1+\varepsilon_{\text{hi}}\big)\, A_{m,i}
\Big),
\label{eq:tetpo-loss}
\end{equation}
with $A_{m,i}$ and all ESS-selection quantities treated as constants w.r.t.\ the
learner gradient. This is a bridge-\emph{weighted} PPO surrogate, not the exact
gradient of the bridge objective: since $\tilde\pi_{\alpha,\theta}$ depends on
$\theta$, $\nabla_\theta\mathbb E_{\tilde\pi_\alpha}[R]=\alpha\,\mathbb E_{\tilde\pi_\alpha}[(R-\mathbb E_{\tilde\pi_\alpha}[R])\nabla_\theta\log\pi_\theta]$,
which \eqref{eq:tetpo-loss} does not reproduce.

\paragraph{Two regimes.}
Per group, \eqref{eq:tetpo-loss} interpolates between two ends. \emph{Healthy}
($\ESS_m(1) \geq \kappa G$): $\alpha_m^\star = 1$ and the group uses groupwise SNIS
coefficients inside the clipped surrogate. \emph{Degenerate}: the bisection returns
$\alpha_m^\star \in [0,1)$, tempering toward the uniform-weights limit
($\alpha_m^\star=0$), which is always feasible. No group ever produces a zero
tempering weight (Proposition~\ref{prop:nocollapse}), distinguishing TETPO from
M2PO's hard masking.

\section{Theoretical Guarantees}
\label{sec:theory}

TETPO's tempering enjoys four guarantees, stated and proved in
Appendix~\ref{app:proofs}; we summarise them here.

\paragraph{Exact bridge target (Thm.~\ref{thm:bridge}).}
Per prompt, the tempered coefficient $\rho^\alpha/Z_\alpha(x)$, with
$Z_\alpha(x)=\mathbb E_\mu[\rho^\alpha]$, is the exact importance ratio of a
geometric-bridge policy $\tilde\pi_\alpha\propto\pi_\theta^\alpha\mu^{1-\alpha}$
against $\mu$; $\alpha=1$ recovers $\pi_\theta$ and $\alpha=0$ recovers $\mu$.
Because $Z_\alpha(x)$ is prompt-dependent it cancels \emph{within} a group but not
across prompts, which is why normalisation is groupwise. This characterises the
population target, not the finite-sample estimator (which is biased; see below).

\paragraph{Well-posed selection (Thm.~\ref{thm:alpha-exist}).}
$\alpha\mapsto\ESS_m(\alpha)$ is continuous and non-increasing with
$\ESS_m(0)=G$, so for every $\kappa\in(0,1]$ the target-ESS exponent
$\alpha_m^\star$ exists (uniquely in the degenerate case) and is found by
bisection on $[0,1]$; there is no need for a floor $\alpha_{\min}$.

\paragraph{ESS-controlled dispersion (Thm.~\ref{thm:variance}).}
The self-normalised weights obey
$\frac1G\sum_i(\tilde w_{m,i}(\alpha_m^\star)-1)^2=G/\ESS_m(\alpha_m^\star)-1\le 1/\kappa-1$,
\emph{independent} of the $\pi_\theta$-to-$\mu$ divergence, giving a
weighted-gradient-norm bound
$\|\hat g_m\|^2\le\frac1\kappa\cdot\frac1G\sum_i\|g_{m,i}\|^2$. This is a
deterministic realised-batch bound, not a bound on the estimator's stochastic
variance; in contrast, ordinary IS has
$\Var_\mu[\rho]=e^{D_2(\pi_\theta\|\mu)}-1$, which grows with the drift
($D_2$ the order-two R\'enyi divergence).

\paragraph{Bounded discrepancy (Thm.~\ref{thm:bias}).}
For fixed $\alpha$ and $|R|\le R_{\max}$, the bridge objective differs from the
on-policy objective by at most $R_{\max}\sqrt{\chi^2(\tilde\pi_\alpha\|\pi_\theta)}$.
As $\chi^2(\tilde\pi_\alpha\|\pi_\theta)$ is non-increasing in $\alpha$
(Prop.~\ref{prop:disc-mono}) while dispersion is non-decreasing, $\alpha_m^\star$
is the least-tempered exponent meeting the ESS target, i.e.\ it minimises this
$\chi^2$ bound among feasible exponents. The finite-sample and adaptation bias of
the self-normalised estimator is not characterised.

\paragraph{No masking (Prop.~\ref{prop:nocollapse}).}
All tempered weights are strictly positive, so tempering adds no explicit
zero-weight mask, unlike M2PO, whose greedy mask can zero an entire trust region
(the instability observed at high staleness). Positivity alone does not force a
nonzero final PPO gradient when the advantage is zero or the clip is active.

\paragraph{Complexity.}
Beyond a vanilla GRPO step, TETPO adds only $O(n_{\text{iter}}\cdot MG)$ work: the
$\log\rho_{m,i}$ are already available, and the $24$-iteration per-group bisections
plus $O(G)$ self-normalisation are negligible next to the per-token
forward/backward passes. Table~\ref{tab:compare} (Appendix~\ref{app:proofs})
contrasts TETPO with prior stabilizers.

\section{Experiments}
\label{sec:exp}

We evaluate TETPO against GRPO~\citep{grpo}, GSPO~\citep{gspo}, GEPO~\citep{gepo},
and M2PO~\citep{m2po2026} on math-reasoning benchmarks under high staleness
($k=32$). Table~\ref{tab:performance-compare} reports avg@16 accuracy (\%): TETPO
attains the best average on both models, improving over the strongest baseline by
$+2.1$ points on Qwen3-Base-1.7B and $+2.3$ on Llama-3.2-3B-Instruct, and,
consistent with the no-collapse guarantee, avoids the
mid-training reward drop seen for GEPO and M2PO. Figure~\ref{fig:curves} shows the
Math~500 trajectories at $k=32$: GEPO collapses sharply (steps $200$, $450$) as
low-ESS batches pass through highly concentrated weights, M2PO peaks then declines
as its mask grows aggressive, and the on-policy baselines GRPO/GSPO plateau lower,
whereas TETPO rises smoothly to the highest reward. Full details are in
Appendix~\ref{app:repro}.

\begin{figure}[t]
\centering
\includegraphics[width=0.72\textwidth]{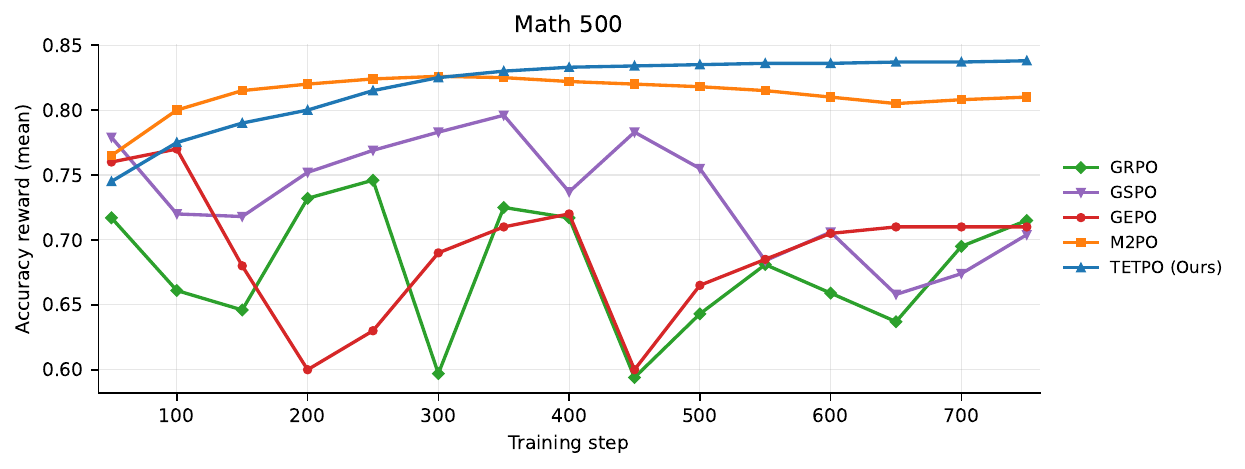}
\vskip -20pt
\caption{Math~500 avg@16 evaluation accuracy versus training step at staleness $k=32$,
for GRPO, GSPO, GEPO, M2PO, and TETPO (ours). GEPO collapses intermittently and
M2PO peaks then declines, whereas TETPO rises smoothly and remains stable.}
\label{fig:curves}
\end{figure}

\begin{table}[t]
\centering
\setlength{\tabcolsep}{3pt}
\caption{Accuracy (\%) across mathematical-reasoning benchmarks, comparing GRPO,
GSPO, GEPO, M2PO, and TETPO. Values are avg@16 accuracy under stochastic
sampling, reported as mean\,$\pm$\,std over
three seeds. Within each model block, the \best{best} result per column is in
bold and the \snd{second best} is underlined.}
\label{tab:performance-compare}
\small
\vskip -8pt
\begin{tabular}{lcccccc}
\toprule
\textbf{Method} & \textbf{AIME24} & \textbf{AIME25} & \textbf{AMC23} & \textbf{AMC24} & \textbf{Math500} & \textbf{Avg.} \\
\midrule
\rowcolor{blockgray}\multicolumn{7}{c}{\textit{Qwen3-Base-1.7B}} \\
\midrule
GRPO~\citep{grpo}     & 15.6\pms{1.1} & 17.5\pms{1.4} & 42.1\pms{0.8} & 21.2\pms{1.0} & 74.6\pms{0.5} & 34.2\pms{0.6} \\
GSPO~\citep{gspo}     & 23.0\pms{0.9} & 18.7\pms{1.2} & 48.0\pms{0.7} & 25.6\pms{0.9} & 79.6\pms{0.4} & 39.0\pms{0.5} \\
GEPO~\citep{gepo}     & 19.5\pms{1.6} & 17.9\pms{1.7} & 43.4\pms{1.2} & 23.1\pms{1.3} & 77.1\pms{0.8} & 36.2\pms{0.9} \\
M2PO~\citep{m2po2026} & \snd{24.1}\pms{0.7} & \snd{24.2}\pms{1.0} & \snd{53.1}\pms{0.6} & \snd{27.5}\pms{0.8} & \snd{82.5}\pms{0.4} & \snd{42.3}\pms{0.5} \\
TETPO (Ours)          & \best{26.6}\pms{0.6} & \best{24.6}\pms{0.8} & \best{54.3}\pms{0.5} & \best{32.8}\pms{0.7} & \best{83.8}\pms{0.3} & \best{44.4}\pms{0.4} \\
\midrule
\rowcolor{blockgray}\multicolumn{7}{c}{\textit{Llama-3.2-3B-Instruct}} \\
\midrule
GRPO~\citep{grpo}     & 1.5\pms{0.7} & 0.0\pms{0.0} & 15.5\pms{1.1} & 5.8\pms{0.9} & 40.2\pms{0.6} & 12.6\pms{0.5} \\
GSPO~\citep{gspo}     & 5.4\pms{0.8} & \snd{0.6}\pms{0.5} & 20.3\pms{0.9} & 9.8\pms{1.0} & 44.6\pms{0.5} & 16.1\pms{0.4} \\
GEPO~\citep{gepo}     & 2.8\pms{0.9} & 0.0\pms{0.0} & \best{22.1}\pms{1.3} & \snd{10.2}\pms{1.1} & \snd{45.9}\pms{0.7} & \snd{16.2}\pms{0.6} \\
M2PO~\citep{m2po2026} & \snd{6.6}\pms{0.6} & 0.3\pms{0.4} & 19.1\pms{0.8} & 8.7\pms{0.7} & 43.5\pms{0.5} & 15.6\pms{0.4} \\
TETPO (Ours)          & \best{8.2}\pms{0.7} & \best{1.6}\pms{0.6} & \snd{21.4}\pms{0.6} & \best{14.9}\pms{0.9} & \best{46.2}\pms{0.4} & \best{18.5}\pms{0.3} \\
\bottomrule
\end{tabular}
\end{table}

\section{Limitations}

Tempering rescales all group weights uniformly; Theorem~\ref{thm:bias} bounds the
fixed-$\alpha$ target discrepancy, but the finite-sample and adaptation bias of the
estimator is not yet characterised. TETPO also uses sequence-level ratios;
combining it with a per-token mask is future work.

\bibliography{ref}

\newpage
\appendix

\section{Theoretical results and proofs}
\label{app:proofs}

\begin{table}[h]
\centering
\small
\caption{TETPO is the only method that (i) operates smoothly, (ii) achieves a
divergence-free (ESS-controlled) weight-dispersion bound, and (iii) provably adds
no explicit zero-weight mask. Second-moment entries use the order-two R\'enyi
divergence $D_2$ ($\chi^2=e^{D_2}-1$).}
\label{tab:compare}
\vskip -6pt
\begin{tabular}{lcccc}
\toprule
\textbf{Method} & \textbf{Granularity} & \textbf{Smoothness} & \textbf{Var./disp.\ bound} & \textbf{Mask-collapse risk} \\
\midrule
GSPO~\citep{gspo} & sequence & PPO clip & $e^{D_2}$ & low \\
GEPO~\citep{gepo} & group & group-exp. & $e^{D_2} - C$ & low \\
M2PO~\citep{m2po2026} & token & hard mask & $R^2 M_2$ & \emph{high} \\
SNIS+ESS & sequence & hard truncate & $1/\beta - 1$ & low \\
\textbf{TETPO (ours)} & sequence & \textbf{smooth} & $\boxed{1/\kappa - 1}$ \emph{(disp.)} & \textbf{none (Prop.~\ref{prop:nocollapse})} \\
\bottomrule
\end{tabular}
\end{table}

\begin{theorem}[Geometric-bridge importance identity]
\label{thm:bridge}
Fix a prompt $x$ with $\pi_\theta(\cdot\mid x),\mu(\cdot\mid x)>0$ on a common
support $\mathcal Y_x$. For $\alpha\in[0,1]$ let
$Z_\alpha(x)\coloneqq\mathbb E_{y\sim\mu(\cdot\mid x)}[\rho(y\mid x)^\alpha]$ and define
the geometric-bridge policy
$\tilde\pi_\alpha(y\mid x)\coloneqq \pi_\theta(y\mid x)^\alpha\,\mu(y\mid x)^{1-\alpha}/Z_\alpha(x)$.
Then $\tilde\pi_\alpha(y\mid x)/\mu(y\mid x)=\rho(y\mid x)^\alpha/Z_\alpha(x)$, so for
every integrable $h$,
\begin{equation}
\mathbb E_{y\sim\tilde\pi_\alpha(\cdot\mid x)}[h(y)]
=\frac{\mathbb E_{y\sim\mu(\cdot\mid x)}[\rho(y\mid x)^\alpha\,h(y)]}
       {\mathbb E_{y\sim\mu(\cdot\mid x)}[\rho(y\mid x)^\alpha]},
\end{equation}
with endpoints $\tilde\pi_0=\mu$ and $\tilde\pi_1=\pi_\theta$.
\end{theorem}

\begin{proof}[Proof of Theorem~\ref{thm:bridge}]
By the definition of $\tilde\pi_\alpha$,
\begin{equation}
\frac{\tilde\pi_\alpha(y\mid x)}{\mu(y\mid x)}
= \frac{\pi_\theta(y\mid x)^\alpha\,\mu(y\mid x)^{1-\alpha}}{Z_\alpha(x)\,\mu(y\mid x)}
= \frac{1}{Z_\alpha(x)}\Big(\frac{\pi_\theta(y\mid x)}{\mu(y\mid x)}\Big)^{\!\alpha}
= \frac{\rho(y\mid x)^\alpha}{Z_\alpha(x)} .
\end{equation}
Hence, for any integrable $h$,
\begin{equation}
\mathbb E_{y\sim\tilde\pi_\alpha(\cdot\mid x)}[h(y)]
= \sum_{y\in\mathcal Y_x}\tilde\pi_\alpha(y\mid x)\,h(y)
= \sum_{y\in\mathcal Y_x}\mu(y\mid x)\,\frac{\rho(y\mid x)^\alpha}{Z_\alpha(x)}\,h(y)
= \frac{\mathbb E_{y\sim\mu(\cdot\mid x)}[\rho(y\mid x)^\alpha\,h(y)]}{Z_\alpha(x)} .
\end{equation}
Taking $h\equiv 1$ gives $Z_\alpha(x)=\mathbb E_{\mu}[\rho^\alpha]$, so
\begin{equation}
\mathbb E_{\tilde\pi_\alpha}[h]
= \frac{\mathbb E_{\mu}[\rho^\alpha\,h]}{\mathbb E_{\mu}[\rho^\alpha]},
\end{equation}
which is the stated ratio-of-expectations identity. The endpoints follow from
$\tilde\pi_0\propto\mu$ and $\tilde\pi_1\propto\pi_\theta$. Finally, $Z_\alpha(x)$
depends on $x$, so it cancels between numerator and denominator only within a
fixed prompt and not under pooled cross-prompt normalisation.
\end{proof}

\begin{theorem}[Existence, monotonicity, uniqueness, bisection]
\label{thm:alpha-exist}
Fix a group with $\rho_1,\dots,\rho_G>0$ and set $f(\alpha)\coloneqq\ESS_m(\alpha)$
on $[0,1]$. Then $f$ is continuous, $f(0)=G$, $f(1)=\ESS(\rho)$, and $f$ is
non-increasing (strictly decreasing on $(0,1]$ when the $\rho_i$ are not all equal).
Hence for every $\kappa\in(0,1]$ the feasible set $\{\alpha:f(\alpha)\ge\kappa G\}$
is a non-empty closed interval $[0,\alpha_m^\star]$: if $f(1)\ge\kappa G$ then
$\alpha_m^\star=1$; if $\kappa<1$ and $f(1)<\kappa G$ then $\alpha_m^\star\in(0,1)$
is the unique root of $f(\alpha)=\kappa G$; and if $\kappa=1$ with non-constant
ratios then $\alpha_m^\star=0$. Bisection on $[0,1]$ brackets $\alpha_m^\star$ to
width $2^{-n_{\text{iter}}}$.
\end{theorem}

\begin{proof}[Proof of Theorem~\ref{thm:alpha-exist}]
Write $z_i \coloneqq \log\rho_i$ and
\begin{equation}
S(a)\coloneqq\sum_{i=1}^{G} e^{a z_i},
\qquad
f(a)=\frac{S(a)^2}{S(2a)} .
\end{equation}
Then $f$ is continuous on $[0,1]$, and
\begin{equation}
f(0)=\frac{G^2}{G}=G,
\qquad
f(1)=\frac{(\sum_i\rho_i)^2}{\sum_i\rho_i^2}=\ESS(\rho) .
\end{equation}
Introduce the exponentially tilted law $p_a(i)\coloneqq e^{a z_i}/S(a)$. Since
$\frac{d}{da}\log S(a)=\mathbb E_{p_a}[z]$,
\begin{equation}
\frac{d}{da}\log f(a)
= 2\,\frac{d}{da}\log S(a)-\frac{d}{da}\log S(2a)
= 2\big(\mathbb E_{p_a}[z]-\mathbb E_{p_{2a}}[z]\big),
\end{equation}
and, moreover,
\begin{equation}
\frac{d}{da}\,\mathbb E_{p_a}[z]=\Var_{p_a}(z)\ge 0 .
\end{equation}
Thus $a\mapsto\mathbb E_{p_a}[z]$ is non-decreasing, so $\frac{d}{da}\log f(a)\le 0$ and
$f$ is non-increasing; it is strictly decreasing on $(0,1]$ whenever the $z_i$ are
not all equal, since then $\Var_{p_a}(z)>0$ for finite $a$. Because
$f(0)=G\ge\kappa G$, the feasible set $\{a\in[0,1]:f(a)\ge\kappa G\}$ is non-empty;
by continuity it is closed, and by monotonicity it is an interval $[0,\alpha_m^\star]$.
The stated cases follow: (i) if $f(1)\ge\kappa G$ then $\alpha_m^\star=1$; (ii) if
$\kappa<1$ and $f(1)<\kappa G$, the intermediate value theorem yields a unique root
of $f(\alpha)=\kappa G$ in $(0,1)$; (iii) if $\kappa=1$ with non-constant ratios,
$f(\alpha)<G$ for all $\alpha>0$, forcing $\alpha_m^\star=0$; (iv) if all ratios are
equal, $f\equiv G$ and $\alpha_m^\star=1$. Bisection halves the bracket at each
iteration.
\end{proof}

\begin{theorem}[ESS-controlled dispersion and weighted-gradient norm]
\label{thm:variance}
For each group $m$ the self-normalised weights $\tilde w_{m,i}(\alpha_m^\star)$
(empirical mean $1$) satisfy the deterministic dispersion bound
\begin{equation}
\frac{1}{G}\sum_{i=1}^{G}\big(\tilde w_{m,i}(\alpha_m^\star)-1\big)^2
= \frac{G}{\ESS_m(\alpha_m^\star)}-1 \le \frac{1}{\kappa}-1,
\end{equation}
independent of any $\pi_\theta$-to-$\mu$ divergence. Consequently, for arbitrary
vectors $g_{m,i}$, the weighted group mean
$\hat g_m\coloneqq\frac{1}{G}\sum_i\tilde w_{m,i}(\alpha_m^\star)g_{m,i}$ obeys
$\|\hat g_m\|^2\le\frac{1}{\kappa}\cdot\frac{1}{G}\sum_i\|g_{m,i}\|^2$, and the
minibatch mean $\hat g\coloneqq\frac{1}{M}\sum_m\hat g_m$ obeys
$\|\hat g\|^2\le\frac{1}{\kappa MG}\sum_{m,i}\|g_{m,i}\|^2$.
\end{theorem}

\begin{proof}[Proof of Theorem~\ref{thm:variance}]
Fix a group $m$ and abbreviate $\tilde w_i\coloneqq\tilde w_{m,i}(\alpha_m^\star)$.
Self-normalisation gives $\frac1G\sum_i\tilde w_i=1$, i.e.\ $\sum_i\tilde w_i=G$.
As ESS is invariant to a common positive rescaling of the weights,
\begin{equation}
\ESS_m(\alpha_m^\star)=\frac{\big(\sum_i\tilde w_i\big)^2}{\sum_i\tilde w_i^2}=\frac{G^2}{\sum_i\tilde w_i^2}
\qquad\Longrightarrow\qquad
\frac1G\sum_{i=1}^{G}\tilde w_i^2=\frac{G}{\ESS_m(\alpha_m^\star)} .
\end{equation}
Therefore, using $\frac1G\sum_i\tilde w_i=1$ and $\ESS_m(\alpha_m^\star)\ge\kappa G$,
\begin{equation}
\frac1G\sum_{i=1}^{G}\big(\tilde w_i-1\big)^2
=\frac1G\sum_i\tilde w_i^2-1
=\frac{G}{\ESS_m(\alpha_m^\star)}-1
\le\frac1\kappa-1 .
\end{equation}
For the gradient bound, the Cauchy--Schwarz inequality gives
\begin{equation}
\Big\|\frac1G\sum_{i=1}^{G}\tilde w_i\,g_i\Big\|^2
\le\Big(\frac1G\sum_i\tilde w_i^2\Big)\Big(\frac1G\sum_i\|g_i\|^2\Big)
\le\frac1\kappa\cdot\frac1G\sum_{i=1}^{G}\|g_i\|^2 .
\end{equation}
Averaging over groups and applying Jensen's inequality,
\begin{equation}
\|\hat g\|^2=\Big\|\frac1M\sum_{m=1}^{M}\hat g_m\Big\|^2
\le\frac1M\sum_{m=1}^{M}\|\hat g_m\|^2
\le\frac{1}{\kappa MG}\sum_{m=1}^{M}\sum_{i=1}^{G}\|g_{m,i}\|^2 .
\end{equation}
\end{proof}

\begin{theorem}[Fixed-$\alpha$ bridge-target discrepancy]
\label{thm:bias}
Fix a prompt $x$ and $\alpha \in [0,1]$ with $\tilde\pi_\alpha(\cdot\mid x)\ll\pi_\theta(\cdot\mid x)$,
and suppose $|R| \le R_{\max}$. Then
\begin{equation}
\big|\mathbb E_{\tilde\pi_\alpha(\cdot\mid x)}[R] - \mathbb E_{\pi_\theta(\cdot\mid x)}[R]\big|
\le R_{\max}\,\sqrt{\chi^2\!\big(\tilde\pi_\alpha(\cdot\mid x)\,\|\,\pi_\theta(\cdot\mid x)\big)},
\end{equation}
where $\chi^2(p\|q) = \mathbb E_q[(p/q-1)^2]$ is the Pearson divergence.
\end{theorem}

\begin{proof}[Proof of Theorem~\ref{thm:bias}]
Let $r_\alpha\coloneqq\tilde\pi_\alpha(\cdot\mid x)/\pi_\theta(\cdot\mid x)$, which is
well defined because $\tilde\pi_\alpha\ll\pi_\theta$, and note
$\mathbb E_{\pi_\theta}[r_\alpha]=1$. Then
\begin{equation}
\mathbb E_{\tilde\pi_\alpha}[R]-\mathbb E_{\pi_\theta}[R]
=\mathbb E_{\pi_\theta}[r_\alpha R]-\mathbb E_{\pi_\theta}[R]
=\mathbb E_{\pi_\theta}\!\big[(r_\alpha-1)R\big] .
\end{equation}
By the Cauchy--Schwarz inequality and $|R|\le R_{\max}$,
\begin{equation}
\big|\mathbb E_{\pi_\theta}[(r_\alpha-1)R]\big|
\le\big(\mathbb E_{\pi_\theta}[(r_\alpha-1)^2]\big)^{1/2}\big(\mathbb E_{\pi_\theta}[R^2]\big)^{1/2}
\le R_{\max}\sqrt{\chi^2\!\big(\tilde\pi_\alpha\,\|\,\pi_\theta\big)},
\end{equation}
where the last step uses $\mathbb E_{\pi_\theta}[(r_\alpha-1)^2]=\chi^2(\tilde\pi_\alpha\|\pi_\theta)$.
The bound concerns the population objective $\mathbb E_{\tilde\pi_\alpha}[R]$ only; it
does not assert unbiasedness of the finite-sample self-normalised estimator, nor
account for the data-adaptive selection of $\alpha_m^\star$.
\end{proof}

\begin{gprop}[Monotonicity of the bridge discrepancy]
\label{prop:disc-mono}
Fix a prompt $x$ with $\mu(\cdot\mid x)$ and $\pi_\theta(\cdot\mid x)$ mutually
absolutely continuous. Let $X\coloneqq\log(\mu(Y\mid x)/\pi_\theta(Y\mid x))$ with
$Y\sim\pi_\theta(\cdot\mid x)$, and suppose $M(t)\coloneqq\mathbb E_{\pi_\theta}[e^{tX}]<\infty$
for $t\in[0,2]$. Then $\alpha\mapsto\chi^2(\tilde\pi_\alpha\|\pi_\theta)$ is
non-increasing on $[0,1]$, and strictly decreasing on $[0,1)$ unless $\mu=\pi_\theta$
almost surely.
\end{gprop}

\begin{proof}
Set $t\coloneqq1-\alpha\in[0,1]$ and $K(t)\coloneqq\log M(t)$. The bridge density
relative to $\pi_\theta$ is
\begin{equation}
\frac{\tilde\pi_\alpha(Y\mid x)}{\pi_\theta(Y\mid x)}=\frac{e^{tX(Y)}}{M(t)},
\end{equation}
so
\begin{equation}
1+\chi^2\!\big(\tilde\pi_\alpha\,\|\,\pi_\theta\big)
=\mathbb E_{\pi_\theta}\!\Big[\Big(\frac{e^{tX}}{M(t)}\Big)^{\!2}\Big]
=\frac{M(2t)}{M(t)^2}
=\exp\!\big(K(2t)-2K(t)\big) .
\end{equation}
Differentiating,
\begin{equation}
\frac{d}{dt}\log\!\big(1+\chi^2(\tilde\pi_\alpha\|\pi_\theta)\big)
=2K'(2t)-2K'(t)\ge 0,
\end{equation}
because $K$ is a cumulant generating function and hence convex, so $K'$ is
non-decreasing. Thus $\chi^2$ is non-decreasing in $t=1-\alpha$, i.e.\
non-increasing in $\alpha$, and strictly so when $X$ is non-degenerate.
\end{proof}

\begin{gprop}[No hard masking by power tempering]
\label{prop:nocollapse}
For every $\alpha \in [0,1]$ and strictly positive $\rho_1,\dots,\rho_G$, the
normalised tempered weights satisfy $\tilde w_{m,i}(\alpha) > 0$. Thus power
tempering never assigns an exactly zero coefficient to a sampled trajectory.
(Strict positivity does not by itself imply a nonzero final PPO gradient: a
sample still contributes zero when its advantage is zero or the clip is active.)
\end{gprop}

\begin{proof}[Proof of Proposition~\ref{prop:nocollapse}]
Fix $\alpha\in[0,1]$. If $\alpha>0$, then $\rho_{m,i}>0$ implies
$\rho_{m,i}^\alpha>0$; if $\alpha=0$, then $\rho_{m,i}^0=1>0$. In either case the
group denominator $\frac1G\sum_j\rho_{m,j}^\alpha>0$, so
\begin{equation}
\tilde w_{m,i}(\alpha)=\frac{\rho_{m,i}^\alpha}{\frac1G\sum_{j=1}^{G}\rho_{m,j}^\alpha}>0 .
\end{equation}
\end{proof}

\section{Experimental Details and Reproducibility}
\label{app:repro}
\sloppy

\paragraph{Models and data.} We post-train \textit{Qwen3-Base-1.7B} (and, in the
second block, \textit{Llama-3.2-3B-Instruct}) with RL. Training prompts are the
SimpleRL-Zoo competition-math split \texttt{SimpleLR/simplelr\_qwen\_level3to5}
(difficulty levels 3--5), loaded from the released \texttt{train.parquet} and
filtered to prompts of at most $768$ tokens; we use the entire filtered split
without subsampling. Evaluation uses five \emph{separate} public benchmarks:
MATH-500 (\texttt{HuggingFaceH4/MATH-500}), AMC23 (\texttt{HanlardResearch/amc23}),
AMC24 (\texttt{rawsh/2024\_AMC12}), AIME24 (\texttt{HanlardResearch/aime\_2024}),
and AIME25 (\texttt{HanlardResearch/aime\_2025}), which come from sources
disjoint from the training corpus, so evaluation items are excluded from training
by construction; we did not run an additional $n$-gram decontamination pass. The
reward is binary correctness computed by \texttt{math-verify} (\texttt{v0.5.2}):
the gold answer is wrapped in \texttt{\textbackslash boxed\{\}} and parsed, the
completion is parsed with boxed-answer priority, and the reward is
$\mathbf 1\{\texttt{verify(answer, gold)}\}\in\{0,1\}$, with no learned reward
model. Rewards are not scaled (\texttt{scale\_rewards=False}); advantages are the
within-group mean-subtracted rewards (GRPO-style).

\paragraph{Optimisation hyperparameters.} All methods share: optimiser AdamW,
learning rate $1{\times}10^{-6}$ with cosine schedule and warmup ratio $0.05$,
gradient clipping $1.0$, KL coefficient $\beta=0$, and $3$ epochs with
evaluation/checkpointing every $64$ steps. Rollouts use $G=8$ completions per
prompt with $M=64$ prompts per optimisation step (effective batch $B=MG=512$
sequences; realised as per-device batch $8$ $\times$ gradient accumulation $8$
across $4$ GPUs), $\max$ prompt length $768$ and $\max$ completion length $4096$.
Decoding for rollouts uses temperature $0.6$, top-$p$ $0.95$, top-$k$ $20$,
$\min$-$p$ $0$. The PPO clip bounds are $\varepsilon_{\text{lo}}=\varepsilon_{\text{hi}}=0.2$.
For TETPO we use target ESS fraction $\kappa=0.85$, floor $\alpha_{\min}=0$, and a
$24$-iteration bisection over $[0,1]$ (${\sim}10^{-7}$ precision).

\paragraph{Staleness protocol.} We use a PipelineRL-style schedule
(\texttt{uniform\_stale\_k}): the sampler serves from a pool of recent learner
snapshots, and each rollout's policy lag is drawn independently as
$\ell\sim\mathrm{Uniform}\{1,\dots,k\}$ with $k=32$. Thus $k$ is the
\emph{maximum} lag and the realised lag is random (not fixed or averaged); the
snapshot pool holds up to $256$ checkpoints. This produces the heterogeneous
staleness that TETPO's per-group adaptation targets.

\paragraph{Importance ratios under truncated decoding.}
Our rollout sampler applies temperature scaling followed by hard top-$k$/top-$p$
truncation. Let $\mathcal S_\mu(h)$ denote the token support retained by this
decoding rule at history $h$. The resulting behaviour distribution is
\[
\mu^{\mathcal S}(a\mid h)
=
\frac{
\mu^{\mathrm{raw}}(a\mid h)
\mathbf 1\{a\in\mathcal S_\mu(h)\}
}{
\sum_{a'\in\mathcal S_\mu(h)}
\mu^{\mathrm{raw}}(a'\mid h)
},
\]
where $\mu^{\mathrm{raw}}$ denotes the temperature-scaled distribution before
top-$k$/top-$p$ truncation. Because $\mu^{\mathcal S}(a\mid h)=0$ outside
$\mathcal S_\mu(h)$ while the learner generally assigns positive probability
there, trajectories generated by $\mu^{\mathcal S}$ cannot provide an exact
importance-sampling estimator for the unrestricted learner policy. An exact
change-of-measure identity is obtained instead for the learner restricted to
the same behaviour support,
\[
\pi_\theta^{\mathcal S}(a\mid h)
=
\frac{
\pi_\theta^{\mathrm{raw}}(a\mid h)
\mathbf 1\{a\in\mathcal S_\mu(h)\}
}{
\sum_{a'\in\mathcal S_\mu(h)}
\pi_\theta^{\mathrm{raw}}(a'\mid h)
},
\]
for which
\[
\rho^{\mathcal S}(y\mid x)
=
\prod_t
\frac{
\pi_\theta^{\mathcal S}(y_t\mid x,y_{<t})
}{
\mu^{\mathcal S}(y_t\mid x,y_{<t})
}
\]
satisfies $\pi_\theta^{\mathcal S}\ll\mu^{\mathcal S}$ and
$\mathbb E_{y\sim\mu^{\mathcal S}}[\rho^{\mathcal S}(y\mid x)]=1$.

In the experiments reported here, however, we do not retain
$\mathcal S_\mu(h)$ or the corresponding truncation normalisers. We record the
log-probability of the sampled token under the pre-truncation,
temperature-scaled behaviour distribution and form the sequence-level proxy
\[
\widehat\rho(y\mid x)
=
\prod_t
\frac{
\pi_\theta^{\mathrm{raw}}(y_t\mid x,y_{<t})
}{
\mu^{\mathrm{raw}}(y_t\mid x,y_{<t})
}.
\]
Consequently, whenever top-$k$/top-$p$ truncation is active,
$\widehat\rho$ is not the exact importance ratio for either the unrestricted
learner policy or the support-restricted policy
$\pi_\theta^{\mathcal S}$. TETPO's empirical ESS and the selected exponent
$\alpha^\star$ are therefore computed from $\widehat\rho$ and should be viewed
as using an approximate likelihood-ratio proxy.

This distinction affects the interpretation of our theoretical results.
The ESS monotonicity, target-ESS feasibility, deterministic weight-dispersion
bound, and positivity results (Theorems~2--3 and Proposition~2) hold for any
strictly positive input ratios and therefore continue to hold exactly when
TETPO is applied to $\widehat\rho$. In contrast, the exact geometric-bridge
identity and the associated bridge-target discrepancy results
(Theorems~1 and~4 and Proposition~1) apply to the exact
support-restricted ratio $\rho^{\mathcal S}$ and do not constitute guarantees
for the approximate ratios used in the present experiments. Logging the
retained support (or otherwise evaluating the post-truncation normalisers)
would permit exact computation of $\rho^{\mathcal S}$; we leave this
implementation and a direct comparison between exact and approximate ratios
to future work.
\paragraph{Evaluation and overhead.} Reported numbers are avg@16 accuracy (\%) on
the five benchmarks above: for each prompt we draw $16$ independent samples with
stochastic decoding (temperature $0.6$, top-$p$ $0.95$, top-$k$ $20$, matching the
rollout settings) and average the binary correctness over those samples. We use
avg@16 rather than single-sample greedy accuracy because several of these
benchmarks are small (AIME24/AIME25 contain $30$ problems each), so a single
greedy decode per prompt gives a very coarse, high-variance estimate. Each
configuration is run with three training seeds, and
Table~\ref{tab:performance-compare} reports mean\,$\pm$\,standard deviation over
those seeds; per-seed results accompany the code. The only computation TETPO
adds to a GRPO step is the per-group $24$-iteration bisection on already-computed
log-ratios ($O(n_{\text{iter}}\,MG)$ elementwise operations) and the groupwise
self-normalisation; these introduce no additional stored activations, and we
observed no measurable change in per-step wall-clock relative to GRPO.

\section{Additional Related Work}
\label{app:related}

\paragraph{RL post-training and LLM agents.}
TETPO targets the off-policy correction that arises in asynchronous RL
post-training of LLMs~\citep{noukhovitch2024arxiv,fu2025areal,wu2025llamarl,piche2025pipelinerl},
where group-relative objectives such as GRPO~\citep{grpo} and its sequence-level
variants~\citep{gspo} are trained on stale rollouts. Beyond single-turn reasoning,
RL is increasingly used to train multi-step LLM agents, where unstable updates are
a central difficulty: hierarchical RL with stable temporal
abstraction~\citep{hamidi2026hierarchical} stabilizes agent training by
structuring decisions over temporally extended actions, and offline imagination
with uncertainty-guided planning~\citep{hamidi2026dreamphase} reduces reliance on
costly online interaction by planning over imagined trajectories. Both settings
reuse data generated under a policy other than the current learner, so the
weight-degeneracy problem studied here also arises there, and TETPO's group-wise
ESS control is complementary to these designs. A related line of work studies
preference-based post-training and distillation, e.g.\ how the choice of models
used to generate rejected responses affects preference distillation at
scale~\citep{cai2026smaller}; such pipelines likewise train on samples drawn from
a distribution that differs from the learner's, which motivates
distribution-mismatch corrections of the kind TETPO provides.

\end{document}